\documentclass[conference,letterpaper]{IEEEtran}

\usepackage[utf8]{inputenc} 
\usepackage[T1]{fontenc}
\usepackage{url}
\usepackage{ifthen}
\usepackage{cite}
\usepackage[cmex10]{amsmath} %

\usepackage[ruled]{algorithm}
\usepackage{colortbl}
\usepackage{booktabs}
\usepackage{multirow}
\usepackage{caption}
\usepackage{float}
\usepackage{multicol}
\usepackage[normalem]{ulem}
\usepackage{algorithmic}
\usepackage{tikz-layers}
\usepackage{mathtools}
\usepackage{graphicx}
\usepackage{url}

\usepackage[
    shortcuts,       %
    acronym,         %
    section=section, %
]{glossaries}
\usepackage{tikz}
\usetikzlibrary{backgrounds,calc,shadings,shapes.arrows,shapes.symbols,shadows,fit,positioning,topaths,hobby}
\usepackage{float}
\usepackage{amssymb}
\usepackage{amsmath}
\usepackage{amsthm}
\usepackage{xspace}
\usepackage[capitalize]{cleveref}
\usepackage{subfig}
\usepackage[inline]{enumitem}
\usepackage{cite}
\usepackage{siunitx}
\usepackage{dsfont}
\usepackage{balance}
\usepackage{subcaption}

\usepackage[
  top=0.73in,       %
  bottom=1.04in,       %
  left=.64in,         %
  right=.64in,        %
]{geometry}

\makeglossaries

\newtheorem{proposition}{Proposition}
\newtheorem{definition}{Definition}

\usepackage{pgfplots}
\pgfplotsset{compat=newest}

\IEEEoverridecommandlockouts

\newcommand{\D}{\ensuremath{\mathcal{D}}}
\newcommand{\Dt}{\ensuremath{\mathcal{D}_\text{test}}}
\newcommand{\Df}{\ensuremath{\D_f}}

\newcommand{\Dp}{\ensuremath{\D_p}}
\newcommand{\unlearn}{\ensuremath{\mathcal{U}}}
\newcommand{\alg}{\ensuremath{\mathcal{A}}}
\newcommand{\algvar}{\ensuremath{\mathcal{A}^\prime}}
\newcommand{\sv}{\ensuremath{\mathcal{S}}}

\newcommand{\distance}[2]{\ensuremath{D}(#1 \Vert #2)}
\newcommand{\weights}{w}
\newcommand{\pweights}[1][\weights]{\ensuremath{P}(#1)}
\newcommand{\verify}[1][\weights]{\ensuremath{\mathcal{V}}(#1)}
\newcommand{\svmmodel}{\ensuremath{w_{\text{SVM}}}}
\newcommand{\fe}{\ensuremath{w_{\text{FE}}}}
\newcommand{\lldense}{\ensuremath{w_{\text{PL}}}}
\newcommand{\totalmodel}{\ensuremath{W}}
\newcommand{\loss}{\ensuremath{\mathcal{L}}}

\newcommand{\sampleidx}{\ensuremath{\ell}}
\newcommand{\sample}[1][\sampleidx]{\ensuremath{\mathbf{x}_{#1}}}
\newcommand{\lab}[1][\sampleidx]{\ensuremath{y_{#1}}}
\newcommand{\embedding}[1][\sampleidx]{\ensuremath{\mathbf{e}_{#1}}}
\newcommand{\topk}[1][k]{\ensuremath{\D_{#1}}}
\newcommand{\kvar}{\ensuremath{k}}
\newcommand{\Dr}[1][\kvar]{\ensuremath{\D_{r}}} %
\newcommand{\nontopk}[1][\kvar]{\ensuremath{\D_{\lnot #1}}}
\newcommand{\freq}[1][\sampleidx]{\ensuremath{f_{#1}}}
\newcommand{\nsamples}{\ensuremath{m}}
\newcommand{\thres}{\ensuremath{\tau}}
\newcommand{\optthres}{\ensuremath{\tau^\star}}
\newcommand{\confidence}[1][\D]{\ensuremath{C(#1)}}
\newcommand{\tpr}{\ensuremath{\text{TPR}}}
\newcommand{\fpr}{\ensuremath{\text{FPR}}}
\newcommand{\define}{\triangleq}

\newcommand{\unlearnedfe}{unlearned FE}
\newcommand{\unlearnedsvm}{unlearned SVM}
\newcommand{\unlearnedmodel}{unlearned FE + SVM}

\newcommand{\algname}{\textsc{MaxRR}\xspace}

\begin{document}
\title{Efficient Machine Unlearning by\\ Model Splitting and Core Sample Selection}

\author{
\IEEEauthorblockN{Maximilian Egger\IEEEauthorrefmark{1}, Rawad~Bitar\IEEEauthorrefmark{1} and Rüdiger Urbanke\IEEEauthorrefmark{2}\\}
\IEEEauthorblockA{\IEEEauthorrefmark{1}%
                   Technical University of Munich, Germany \{maximilian.egger, rawad.bitar\}@tum.de}
                   
\IEEEauthorblockA{\IEEEauthorrefmark{2}%
                   École Polytechnique Fédérale de Lausanne, Switzerland \{rudiger.urbanke\}@epfl.ch}
\thanks{This project is funded by DFG (German Research Foundation) projects under Grant Agreement Nos. BI 2492/1-1 and WA 3907/7-1.}
\thanks{Part of the work was done when RB and ME visited RU at EPFL supported in parts by EuroTech Visiting Researcher Programme grants.}%
}

\maketitle
\IEEEpeerreviewmaketitle

\begin{abstract}
Machine unlearning is essential for meeting legal obligations such as the \emph{right to be forgotten}, which requires the removal of specific data from machine learning models upon request. While several approaches to unlearning have been proposed, existing solutions often struggle with efficiency and, more critically, with the verification of unlearning—particularly in the case of weak unlearning guarantees, where verification remains an open challenge. We introduce a generalized variant of the standard unlearning metric that enables more efficient and precise unlearning strategies. We also present an \emph{unlearning-aware} training procedure that, in many cases, allows for exact unlearning. We term our approach \algname. When exact unlearning is not feasible, \algname still supports efficient unlearning with properties closely matching those achieved through full retraining.
\end{abstract}

\section{Introduction}
Driven by the General Data Protection Regulation (GDPR)~\cite{GDPR} and the California Consumer Privacy Act (CCPA)~\cite{CCPA}, the \emph{right to be forgotten} has recently gained significant importance. Machine unlearning has emerged as a key concept to address this requirement. When a model has been trained using sensitive data from individuals or organizations, any subsequent unlearning request should ensure that the influence of the specified data is effectively removed from the trained model. The resulting unlearned model should ideally behave as if it had been trained from scratch on the remaining data, excluding the specified sensitive data.

Unlearning approaches can be broadly categorized into \emph{exact} and \emph{approximate} methods. Exact procedures have been proposed, for instance, for deep neural networks \cite{bourtoule2021machine}, for graph-based models \cite{chen2022graph}, k-means clustering \cite{ginart2019making}, and via training sample transformations \cite{cao2015towards}. Approximate unlearning techniques based on influence functions were explored in \cite{guo2020certified,warnecke2021machine}, followed by more efficient strategies \cite{suriyakumar2022algorithms,mehta2022deep}. Re-optimization-based approximate unlearning was introduced in \cite{golatkar2020eternal}, aiming to make the unlearned model indistinguishable from a retrained one, while gradient-based methods appeared in \cite{wu2020deltagrad,cao2023fedrecover}. Additionally, machine unlearning can be split into \emph{data-oriented} methods (e.g., through partitioning or data modification) and \emph{model-oriented} methods (e.g., model reset or structural changes). Comprehensive overviews are available in recent surveys \cite{nguyen2022survey,xu2023machine,wang2024machine,liu2025survey,liu2025threats}. Unlearning in federated settings has been explored in \cite{pan2023machine}. An information-theoretic approach for smoothing gradients near the forgotten sample was proposed in \cite{foster2025an}, where the decision boundary behavior was also analyzed. Attempts to reduce the gap between exact and approximate unlearning were made in \cite{jia2023model} through weight sparsification.

Beyond privacy, two major challenges remain for unlearning mechanisms: \emph{efficiency} and \emph{verifiability}. Efficient unlearning is critical for practical scalability, with attempts such as \cite{tarun2024fast} addressing this. Verifiability ensures that unlearning requests have been properly executed. Certified data removal with theoretical guarantees was developed in \cite{guo2019certified}. Membership inference attacks (MIAs) are widely used to verify unlearning, including entropy-based methods \cite{salem2019ml} and general frameworks \cite{ye2022enhanced}; see \cite{niu2024survey} for a detailed MIA survey. Rigorous verification frameworks based on hypothesis testing and backdoor detection were proposed in \cite{sommer2020towards}, while information-theoretic techniques using layer-wise information differences were introduced in \cite{jeon2024information}. However, verification remains fragile \cite{zhang2024verification}, and the limitations of MIAs in general verification were highlighted in \cite{zhang2024membership}, revealing the lack of robust alternatives. Privacy auditing using canaries in LLMs was proposed in \cite{panda2025privacy}. However, this approach is highly specific and unsuitable for general machine learning models.

For individual privacy and the final performance on unseen data, well-generalizing models are desirable. However, such models are also less prone to MIA \cite{li2021membership}, further hindering verification. A similar connection applies to adversarial robustness \cite{madry2018towards}, which was jointly studied with MIA in \cite{ding2024regularization}. It was shown in \cite{song2019membership} that adversarially robust training is more robust to MIA. Even honest providers face difficulties in verifying approximate unlearning due to the lack of reliable verification methods across diverse architectures and training procedures.

To address these challenges, we propose a generalized notion of machine unlearning that improves the efficiency of unlearning procedures. Inspired by the observation that forgetting unimportant data may not significantly impact the model \cite{wang2024machine}, we introduce an unlearning-aware training process. In many cases, this eliminates the need for post-training unlearning altogether, while still providing strong guarantees. When unlearning is required, our approach enables a simple and efficient method to filter out the influence of the target data. Our framework, called \algname, is composed of unlearning-aware training and efficient unlearning, and supports reliable verification via confidence-optimized MIAs. It relies on decomposing the model into a standard feature extractor and a support vector machine, allowing us to exploit the relative importance of training samples during the initial model construction.

\section{System Model and Preliminaries}
We consider a dataset $\D$ consisting of $\nsamples$ samples indexed by $\sampleidx \in [\nsamples]$. Each sample comprises features $\sample[\sampleidx]$ and a corresponding label $\lab[\sampleidx]$. A model $\alg(\D)$ is trained on this dataset using a learning algorithm $\alg$. Within $\D$, a subset of samples $\Df \subset \D$ is later requested to be unlearned (forgotten). The goal is to design an unlearning method $\unlearn$ that, given the trained model $\alg(\D)$, the original training data $\D$, and the data to be forgotten $\Df$, produces an unlearned model $\unlearn(\alg(\D), \D, \Df)$ which no longer retains information about $\Df$.

The prevailing notion of successful unlearning requires that the unlearned model be indistinguishable from a model trained from scratch on the reduced dataset $\D \setminus \Df$ using the same algorithm $\alg$. Formally, for $\Df \subset \D$, this means $\unlearn(\alg(\D), \D, \Df) \approx \alg(\D \setminus \Df)$. In the literature, unlearning methods are typically categorized as either exact or approximate, often defined rigorously over a probability space of model weights or the function space.

Let $\weights$ represent the weights of a model, and $\pweights[\weights]$ denote the probability distribution over these weights. A distance measure $\distance{\pweights[\weights_1]}{\pweights[\weights_2]}$ is used to quantify the difference between two such distributions. Under this formulation, \emph{exact unlearning} is achieved when $\distance{\unlearn(\alg(\D), \D, \Df)}{\alg(\D \setminus \Df)} = 0$. \emph{Approximate unlearning} relaxes this condition by allowing a small tolerance $\epsilon$ in the difference measure \cite{thudi2022unrolling}, such as the Kullback-Leibler divergence employed in \cite{golatkar2020eternal}. The magnitude of $\epsilon$ further distinguishes between \emph{strong} and \emph{weak} approximate unlearning schemes.
Alternatively, the distribution can be defined over the model function space rather than its weights \cite{thudi2022unrolling}. Non-probabilistic definitions have also been explored, e.g., using the $L2$-distance between model parameters~\cite{wu2020deltagrad}.

We propose a generalization of the standard unlearning definition. Rather than comparing only to $\alg(\D \setminus \Df)$, we allow comparisons to any model obtained by training on an arbitrary subset $\Dp \subset \D$ such that $\Df \cap \Dp = \emptyset$, using any learning algorithm $\algvar$ that is independent of $\Df$ and potentially distinct from $\alg$. We adopt the probabilistic framework as in \cite{golatkar2020eternal,nguyen2022survey,xu2023machine}, and formally state this generalized notion of exact unlearning.

\begin{definition}[Generalized exact unlearning] \label{def:relaxed_exact_unlearning}
A model $\alg(\D)$ is said to unlearn $\Df\subset\D$ if there exists an algorithm $\algvar$ independent of $\Df$, and a data subset $\Dp \subseteq \D\setminus \Df$ such that
\begin{align*}
    \distance{\pweights[\unlearn(\alg(\D), \D, \Df)]}{\pweights[\algvar(\Dp)]} = 0.
\end{align*}
\end{definition}
This generalized notion offers greater flexibility for service providers responding to unlearning requests for specific portions of the data. As we will see, with this notion, we can achieve exact unlearning with minimal effort in a substantial fraction of cases. To further reduce the associated effort, an approximate variant of this notion permits a slack $\epsilon$ in the unlearning measure of \cref{def:relaxed_exact_unlearning}, analogous to the approximate relaxations of exact unlearning studied in the literature.

To provide efficient and effective unlearning guarantees according to \cref{def:relaxed_exact_unlearning}, we make use of the structure of support vector machines (SVMs). We briefly revisit them using their dual form. Let $C$ be a regularization parameter, and $K(\sample[\sampleidx], \sample[\sampleidx^\prime])$ a kernel (linear herein). The dual soft-margin SVM optimization problem for binary classification on the data $\D$ is
\begin{align}
\max_{\alpha_1, \cdots, \alpha_\nsamples} \quad & \sum_{\sampleidx=1}^{\nsamples} \alpha_{\sampleidx} - \frac{1}{2} \sum_{\sampleidx=1}^{\nsamples} \sum_{\sampleidx^\prime=1}^{\nsamples} \alpha_{\sampleidx} \alpha_{\sampleidx^\prime} \lab[\sampleidx] \lab[\sampleidx^\prime] K(\sample[\sampleidx], \sample[\sampleidx^\prime]), \label{eq:dual_svm}
\end{align}
subject to $0 \leq \alpha_{\sampleidx} \leq C$, $\sampleidx = 1, \dots, \nsamples$, and $\sum_{\sampleidx=1}^{\nsamples} \alpha_{\sampleidx} \lab[\sampleidx] = 0$, where the $\alpha_{\sampleidx}$ are the dual variables. Let $\sv$ be the set of support vectors, then $\alpha_{\sampleidx} > 0$ iff $\sampleidx \in \sv$. The decision function is
\begin{align*}
f(x) = \text{sign}\Big(\sum_{\sampleidx \in \sv} \alpha_{\sampleidx} \lab[\sampleidx] K(\sample[\sampleidx], x) + b\Big),
\end{align*}%
where $\text{sgin}(x)$ is equal to $1$ if $x\geq 0$ and $-1$ otherwise, and%
\[
b = \frac{1}{\vert \sv \vert} \sum_{\sampleidx = 1}^{\vert \sv \vert} \left( \lab[\sampleidx] - \sum_{\sampleidx^\prime \in \sv} \alpha_{\sampleidx^\prime} \lab[\sampleidx^\prime] K(\sample[\sampleidx], \sample[\sampleidx^\prime]) \right).%
\]
  
For multi-class classification, we use the One-vs-Rest (OvR) method, combining the predictions of multiple binary SVMs into a multi-class prediction.

The verification of unlearned machine learning models plays a key role in privacy auditing. A verification algorithm $\verify[\cdot]$ should be able to distinguish between a model that retains information about a subset $\Df$—such as $\alg(\D)$—and one that has successfully unlearned it, i.e., $\verify[\alg(\D)] \neq \verify[\alg(\D \setminus \Df)]$. We use $\sv$ to denote support vectors and the corresponding indices interchangeably; the meaning will be clear from context.

\section{Illustrative Example of Efficient Unlearning}
To illustrate our core ideas, we begin with a simple case: a support vector machine (SVM) for binary classification. We later generalize these ideas to develop an unlearning method applicable to more complex models. Suppose an SVM with a linear kernel is trained on the full dataset $\D$, yielding a model $\svmmodel$. This predictive model depends solely on the support vectors, which we denote by $\sv \subset \D$.

Now consider an unlearning request for any subset $\Df \subseteq \D \setminus \sv$. In this case, the model $\svmmodel$ is already exactly unlearned in the sense of \cref{def:relaxed_exact_unlearning}, as it is entirely determined by the remaining data $\Dp = \sv$ that satisfies $\Dp\cap\Df = \emptyset$. Moreover, there exists a learning algorithm that, when applied to $\sv$, produces the same model as training on the full dataset $\D$. This is formally captured by the following result.

\begin{figure}[t]
\centering
\begin{tikzpicture}[x=1.5cm, y=0.95cm,
  neuron/.style={circle, draw, minimum size=0.55cm, thick, inner sep=0pt},
  layerbox/.style={draw=black, thick, rounded corners, inner sep=0.15cm}
]

  \node[neuron] (I1) at (0,0.5) {};
  \node[neuron] (I2) at (0,-0.5) {};
  \node at (-0.5, 0) {\small Input};

  \node[neuron] (H11) at (1,1.0) {};
  \node[neuron] (H12) at (1,0.3) {};
  \node[neuron] (H13) at (1,-0.3) {};
  \node[neuron] (H14) at (1,-1.0) {};

  \node[neuron] (H21) at (2,0.7) {};
  \node[neuron] (H22) at (2,0.0) {};
  \node[neuron] (H23) at (2,-0.7) {};

  \node[neuron] (H31) at (3,0.5) {};
  \node[neuron] (H32) at (3,-0.5) {};

  \node[rectangle, draw=black, thick, minimum width=1.1cm, minimum height=1.1cm, align=center] (SVM) at (4.2,0) {\small $\svmmodel$};
  \node at (4.2,.9) {\small SVM trained on $\mathcal{D}$};

  \foreach \i in {I1,I2}
    \foreach \h in {H11,H12,H13,H14}
      \draw[->, thick] (\i) -- (\h);

  \foreach \h in {H11,H12,H13,H14}
    \foreach \hh in {H21,H22,H23}
      \draw[->, thick] (\h) -- (\hh);

  \foreach \h in {H21,H22,H23}
    \foreach \hh in {H31,H32}
      \draw[->, thick] (\h) -- (\hh);

  \foreach \h in {H31,H32}
    \draw[->, thick] (\h) -- (SVM);

  \node[layerbox, fit=(H11)(H12)(H13)(H14)(H21)(H22)(H23)(H31)(H32)] (featbox) {};
  \node at (2,1.8) {\small Feature Extractor trained on $\topk$};
  \node at (2.5,1.2) {\small $\fe$};

\end{tikzpicture}
\vspace{-.2cm}
\caption{Unlearning-aware model architecture in \algname, $\D$ is the entire training set, $\topk$ are the topmost $\kvar$ important samples.}
\label{fig:architecture}
\vspace{-.3cm}
\end{figure}
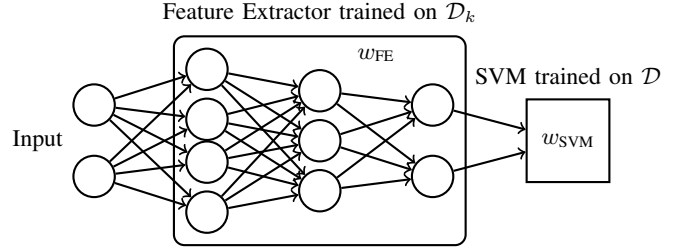

\begin{proposition}
Let $\svmmodel$ be a linear SVM trained on the dataset $\D$, and let $\sv \subset \D$ denote the resulting support vectors. For any $\Df \subseteq \D \setminus \sv$, retraining the SVM on $\D \setminus \Df$ yields the same model, i.e., the resulting SVM is equivalent to $\svmmodel$.
\end{proposition}
\begin{proof}
The proof follows from the dual formulation of the SVM in \eqref{eq:dual_svm}. Let $\sv$ denote the support vectors corresponding to the optimal solution of \eqref{eq:dual_svm} on the full dataset $\D$. This solution remains feasible for the dual problem defined on $\D \setminus \Df$, as the removal of $\Df$ simply relaxes the feasibility region by eliminating constraints. However, this does not guarantee that it remains optimal.

Suppose, for contradiction, that the optimal solution on $\D \setminus \Df$ differs from that on $\D$. Then, this new solution is also feasible for the dual problem on $\D$—by setting $\alpha_\sampleidx = 0$ for all $\sampleidx$ corresponding to samples in $\Df$. This would imply that a better solution exists for the original problem on $\D$, contradicting the assumption that the original solution, fully determined by $\sv$, was optimal.
\end{proof}

\begin{algorithm}[t]
\caption{Model Splitting for $\fe \circ \svmmodel$}
\label{alg:training_process}
\begin{algorithmic}[1]
\REQUIRE Training data $\D$, feature extractor $\fe$, prediction layer $\lldense$, SVM $\svmmodel$, loss function $\loss$
\STATE Train $\fe \circ \lldense$ with backpropagation under loss $\mathcal{L}$
\STATE Predict feature embeddings $\embedding[\sampleidx] = \fe(\sample[\sampleidx]), \forall \sampleidx \in [\vert \D \vert]$.
\STATE Fit an $\svmmodel$ on samples $\{\embedding[\sampleidx], \lab\}_{\sampleidx = 1}^{\vert \D \vert}$
\RETURN Final model $\fe \circ \svmmodel$
\end{algorithmic}
\end{algorithm}

Given an unlearning request for a subset of samples $\Df \subset \sv$, the SVM can, in principle, be retrained from scratch on $\D \setminus \Df$, corresponding to traditional exact unlearning. However, to reduce retraining complexity, it is possible to exploit the information about the support vectors identified during the initial training. In particular, an unlearned SVM can be retrained using only the remaining support vectors, $\Dp = \sv \setminus \Df$. This approach offers significant computational savings, albeit potentially reducing the model's performance.

These observations are enabled by the structure of SVMs; in particular, the fact that the final model depends solely on the support vectors, which typically represent a small subset of the training data. For all other samples, no retraining is necessary if they are later requested to be unlearned. In cases where the samples to be forgotten are support vectors, the cost of unlearning can still be reduced by reusing the identity of the remaining support vectors as prior knowledge.

\section{Efficient Unlearning via Model Splitting and Core Sample Selection} \label{sec:main_section}
We introduce an unlearning strategy for general neural networks grounded in two key principles: (i) \emph{model splitting} and (ii) \emph{core sample selection}. 

The core concept of \emph{model splitting} involves decomposing a neural network into two parts: a \emph{feature extractor} (FE) and a \emph{final prediction layer} (PL). In our approach, we substitute the PL with a linear SVM, as illustrated in \cref{fig:architecture}.
This leads to a model architecture defined as the composition $\fe \circ \svmmodel$, where the feature extractor $\fe$ transforms input data into feature embeddings, and the SVM $\svmmodel$ performs classification using these features. A detailed overview of this model splitting process is provided in \cref{alg:training_process}. Given a dataset $\mathcal{D}$, we initially train a conventional neural network in an end-to-end manner. This network is naturally expressed as a composition of two functions: a feature extractor $\fe$ and a final dense layer $\lldense$, resulting in the original model form $\fe \circ \lldense$.
After training, we replace the dense prediction layer $\lldense$ with a linear SVM. To do this, we use the trained feature extractor $\fe$ to compute embeddings for the training data, which are then used to train the SVM classifier $\svmmodel$. This produces the final model $\fe \circ \svmmodel$.

As elaborated in Section~\ref{sec:exact_and_approximate_unlearning}, this model splitting framework enables \emph{approximate} unlearning of all training samples by leveraging techniques analogous to those employed in unlearning for standalone SVM systems, as discussed previously. 

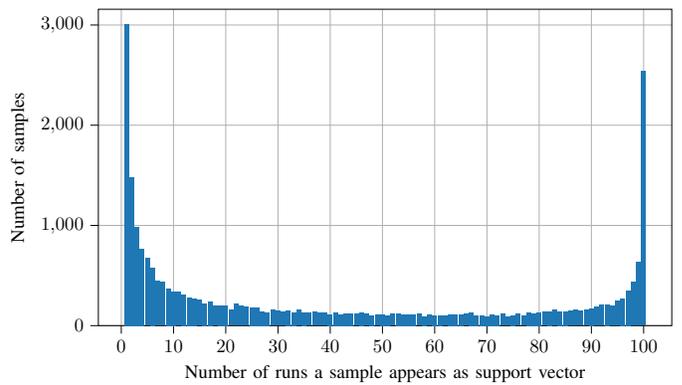
\begin{figure}
    \centering
    \resizebox{\linewidth}{!}{\begin{tikzpicture}

\definecolor{color0}{rgb}{0.12156862745098,0.466666666666667,0.705882352941177}

\begin{axis}[
tick align=outside,
tick pos=left,
x grid style={white!69.0196078431373!black},
xlabel={Number of runs a sample appears as support vector},
xmajorgrids,
xmin=-4.39, xmax=105.39,
xtick style={color=black},
y grid style={white!69.0196078431373!black},
ylabel={Number of samples},
ymajorgrids,
ymin=0, ymax=3157.35,
ytick style={color=black}
]
\draw[draw=none,fill=color0] (axis cs:91.6,0) rectangle (axis cs:92.4,207);
\draw[draw=none,fill=color0] (axis cs:99.6,0) rectangle (axis cs:100.4,2540);
\draw[draw=none,fill=color0] (axis cs:39.6,0) rectangle (axis cs:40.4,111);
\draw[draw=none,fill=color0] (axis cs:33.6,0) rectangle (axis cs:34.4,155);
\draw[draw=none,fill=color0] (axis cs:57.6,0) rectangle (axis cs:58.4,93);
\draw[draw=none,fill=color0] (axis cs:85.6,0) rectangle (axis cs:86.4,147);
\draw[draw=none,fill=color0] (axis cs:21.6,0) rectangle (axis cs:22.4,215);
\draw[draw=none,fill=color0] (axis cs:90.6,0) rectangle (axis cs:91.4,186);
\draw[draw=none,fill=color0] (axis cs:43.6,0) rectangle (axis cs:44.4,122);
\draw[draw=none,fill=color0] (axis cs:98.6,0) rectangle (axis cs:99.4,634);
\draw[draw=none,fill=color0] (axis cs:41.6,0) rectangle (axis cs:42.4,114);
\draw[draw=none,fill=color0] (axis cs:18.6,0) rectangle (axis cs:19.4,200);
\draw[draw=none,fill=color0] (axis cs:95.6,0) rectangle (axis cs:96.4,268);
\draw[draw=none,fill=color0] (axis cs:94.6,0) rectangle (axis cs:95.4,248);
\draw[draw=none,fill=color0] (axis cs:13.6,0) rectangle (axis cs:14.4,269);
\draw[draw=none,fill=color0] (axis cs:23.6,0) rectangle (axis cs:24.4,185);
\draw[draw=none,fill=color0] (axis cs:89.6,0) rectangle (axis cs:90.4,172);
\draw[draw=none,fill=color0] (axis cs:92.6,0) rectangle (axis cs:93.4,212);
\draw[draw=none,fill=color0] (axis cs:88.6,0) rectangle (axis cs:89.4,158);
\draw[draw=none,fill=color0] (axis cs:83.6,0) rectangle (axis cs:84.4,140);
\draw[draw=none,fill=color0] (axis cs:76.6,0) rectangle (axis cs:77.4,96);
\draw[draw=none,fill=color0] (axis cs:82.6,0) rectangle (axis cs:83.4,157);
\draw[draw=none,fill=color0] (axis cs:32.6,0) rectangle (axis cs:33.4,133);
\draw[draw=none,fill=color0] (axis cs:96.6,0) rectangle (axis cs:97.4,351);
\draw[draw=none,fill=color0] (axis cs:60.6,0) rectangle (axis cs:61.4,104);
\draw[draw=none,fill=color0] (axis cs:29.6,0) rectangle (axis cs:30.4,146);
\draw[draw=none,fill=color0] (axis cs:79.6,0) rectangle (axis cs:80.4,131);
\draw[draw=none,fill=color0] (axis cs:58.6,0) rectangle (axis cs:59.4,106);
\draw[draw=none,fill=color0] (axis cs:64.6,0) rectangle (axis cs:65.4,112);
\draw[draw=none,fill=color0] (axis cs:24.6,0) rectangle (axis cs:25.4,179);
\draw[draw=none,fill=color0] (axis cs:40.6,0) rectangle (axis cs:41.4,129);
\draw[draw=none,fill=color0] (axis cs:81.6,0) rectangle (axis cs:82.4,141);
\draw[draw=none,fill=color0] (axis cs:97.6,0) rectangle (axis cs:98.4,436);
\draw[draw=none,fill=color0] (axis cs:61.6,0) rectangle (axis cs:62.4,96);
\draw[draw=none,fill=color0] (axis cs:62.6,0) rectangle (axis cs:63.4,114);
\draw[draw=none,fill=color0] (axis cs:49.6,0) rectangle (axis cs:50.4,114);
\draw[draw=none,fill=color0] (axis cs:93.6,0) rectangle (axis cs:94.4,197);
\draw[draw=none,fill=color0] (axis cs:53.6,0) rectangle (axis cs:54.4,114);
\draw[draw=none,fill=color0] (axis cs:77.6,0) rectangle (axis cs:78.4,132);
\draw[draw=none,fill=color0] (axis cs:22.6,0) rectangle (axis cs:23.4,196);
\draw[draw=none,fill=color0] (axis cs:84.6,0) rectangle (axis cs:85.4,143);
\draw[draw=none,fill=color0] (axis cs:72.6,0) rectangle (axis cs:73.4,115);
\draw[draw=none,fill=color0] (axis cs:51.6,0) rectangle (axis cs:52.4,115);
\draw[draw=none,fill=color0] (axis cs:63.6,0) rectangle (axis cs:64.4,114);
\draw[draw=none,fill=color0] (axis cs:45.6,0) rectangle (axis cs:46.4,130);
\draw[draw=none,fill=color0] (axis cs:11.6,0) rectangle (axis cs:12.4,313);
\draw[draw=none,fill=color0] (axis cs:1.6,0) rectangle (axis cs:2.4,1484);
\draw[draw=none,fill=color0] (axis cs:78.6,0) rectangle (axis cs:79.4,121);
\draw[draw=none,fill=color0] (axis cs:65.6,0) rectangle (axis cs:66.4,116);
\draw[draw=none,fill=color0] (axis cs:86.6,0) rectangle (axis cs:87.4,156);
\draw[draw=none,fill=color0] (axis cs:10.6,0) rectangle (axis cs:11.4,335);
\draw[draw=none,fill=color0] (axis cs:25.6,0) rectangle (axis cs:26.4,184);
\draw[draw=none,fill=color0] (axis cs:74.6,0) rectangle (axis cs:75.4,104);
\draw[draw=none,fill=color0] (axis cs:69.6,0) rectangle (axis cs:70.4,95);
\draw[draw=none,fill=color0] (axis cs:3.6,0) rectangle (axis cs:4.4,769);
\draw[draw=none,fill=color0] (axis cs:34.6,0) rectangle (axis cs:35.4,129);
\draw[draw=none,fill=color0] (axis cs:59.6,0) rectangle (axis cs:60.4,100);
\draw[draw=none,fill=color0] (axis cs:55.6,0) rectangle (axis cs:56.4,113);
\draw[draw=none,fill=color0] (axis cs:68.6,0) rectangle (axis cs:69.4,105);
\draw[draw=none,fill=color0] (axis cs:38.6,0) rectangle (axis cs:39.4,126);
\draw[draw=none,fill=color0] (axis cs:44.6,0) rectangle (axis cs:45.4,118);
\draw[draw=none,fill=color0] (axis cs:80.6,0) rectangle (axis cs:81.4,135);
\draw[draw=none,fill=color0] (axis cs:73.6,0) rectangle (axis cs:74.4,87);
\draw[draw=none,fill=color0] (axis cs:67.6,0) rectangle (axis cs:68.4,101);
\draw[draw=none,fill=color0] (axis cs:19.6,0) rectangle (axis cs:20.4,201);
\draw[draw=none,fill=color0] (axis cs:87.6,0) rectangle (axis cs:88.4,147);
\draw[draw=none,fill=color0] (axis cs:14.6,0) rectangle (axis cs:15.4,256);
\draw[draw=none,fill=color0] (axis cs:8.6,0) rectangle (axis cs:9.4,372);
\draw[draw=none,fill=color0] (axis cs:48.6,0) rectangle (axis cs:49.4,106);
\draw[draw=none,fill=color0] (axis cs:56.6,0) rectangle (axis cs:57.4,115);
\draw[draw=none,fill=color0] (axis cs:37.6,0) rectangle (axis cs:38.4,125);
\draw[draw=none,fill=color0] (axis cs:70.6,0) rectangle (axis cs:71.4,112);
\draw[draw=none,fill=color0] (axis cs:75.6,0) rectangle (axis cs:76.4,118);
\draw[draw=none,fill=color0] (axis cs:66.6,0) rectangle (axis cs:67.4,132);
\draw[draw=none,fill=color0] (axis cs:4.6,0) rectangle (axis cs:5.4,672);
\draw[draw=none,fill=color0] (axis cs:26.6,0) rectangle (axis cs:27.4,141);
\draw[draw=none,fill=color0] (axis cs:71.6,0) rectangle (axis cs:72.4,103);
\draw[draw=none,fill=color0] (axis cs:9.6,0) rectangle (axis cs:10.4,339);
\draw[draw=none,fill=color0] (axis cs:35.6,0) rectangle (axis cs:36.4,126);
\draw[draw=none,fill=color0] (axis cs:6.6,0) rectangle (axis cs:7.4,443);
\draw[draw=none,fill=color0] (axis cs:42.6,0) rectangle (axis cs:43.4,124);
\draw[draw=none,fill=color0] (axis cs:15.6,0) rectangle (axis cs:16.4,219);
\draw[draw=none,fill=color0] (axis cs:2.6,0) rectangle (axis cs:3.4,985);
\draw[draw=none,fill=color0] (axis cs:12.6,0) rectangle (axis cs:13.4,278);
\draw[draw=none,fill=color0] (axis cs:52.6,0) rectangle (axis cs:53.4,123);
\draw[draw=none,fill=color0] (axis cs:20.6,0) rectangle (axis cs:21.4,158);
\draw[draw=none,fill=color0] (axis cs:28.6,0) rectangle (axis cs:29.4,155);
\draw[draw=none,fill=color0] (axis cs:31.6,0) rectangle (axis cs:32.4,152);
\draw[draw=none,fill=color0] (axis cs:27.6,0) rectangle (axis cs:28.4,132);
\draw[draw=none,fill=color0] (axis cs:17.6,0) rectangle (axis cs:18.4,196);
\draw[draw=none,fill=color0] (axis cs:50.6,0) rectangle (axis cs:51.4,104);
\draw[draw=none,fill=color0] (axis cs:36.6,0) rectangle (axis cs:37.4,135);
\draw[draw=none,fill=color0] (axis cs:7.6,0) rectangle (axis cs:8.4,435);
\draw[draw=none,fill=color0] (axis cs:54.6,0) rectangle (axis cs:55.4,111);
\draw[draw=none,fill=color0] (axis cs:46.6,0) rectangle (axis cs:47.4,115);
\draw[draw=none,fill=color0] (axis cs:0.600000000000001,0) rectangle (axis cs:1.4,3007);
\draw[draw=none,fill=color0] (axis cs:47.6,0) rectangle (axis cs:48.4,105);
\draw[draw=none,fill=color0] (axis cs:16.6,0) rectangle (axis cs:17.4,243);
\draw[draw=none,fill=color0] (axis cs:5.6,0) rectangle (axis cs:6.4,578);
\draw[draw=none,fill=color0] (axis cs:30.6,0) rectangle (axis cs:31.4,137);
\end{axis}

\end{tikzpicture}}
    \caption{Comparison of empirical support vector frequency across 100 training runs.} %
    \label{fig:histogram-svs-fmnist}
    \vspace{-.3cm}
\end{figure}

To facilitate \emph{exact} unlearning for a significant portion of the dataset, we propose the concept of \emph{core sample selection}. In this approach, the feature extractor is trained using only a carefully chosen subset of the data—those samples deemed most influential—while the SVM classifier is trained on the entire dataset. We refer to this influential subset as the set of \emph{core samples}. As we will demonstrate, our unlearning method enables efficient and exact unlearning of any non-core sample, which motivates the search for a core set that is as small as possible, without substantially compromising overall model performance.

This naturally raises the question: which subset of the training data is most valuable for learning the feature extractor? Empirical results indicate that omitting support vectors significantly impacts performance, whereas removing non-support vectors has a much smaller effect. Based on this observation, we define core samples by estimating the empirical probability of a sample being selected as a support vector in the final SVM layer across multiple training runs.

Our method for selecting core samples requires several training runs at the start, and unlearning potentially requires retraining the SVM layer, as described next, adding computational complexity to the unlearning algorithm. Therefore, investigating further whether these two steps can be done even more efficiently is an interesting topic for future research.

We explain \algname using numerical experiments. We repeatedly train the composed model $\fe \circ \svmmodel$ on Fashion MNIST\footnote{We also conduct additional experiments on the MNIST dataset, though we observe only minor changes in accuracy—even under extensive unlearning—likely due to the dataset's simplicity.} as dataset $\D, \vert \D \vert = 60 \cdot 10^3$, using \cref{alg:training_process} and the well-known LeNet-5 architecture~\cite{lecun1998gradient} (cf. \cref{app:architecture} for details). Over $100$ independent runs with different random seeds, we track which samples appear as support vectors $\sv$ in each run and record the frequency with which each training sample is selected as a support vector. The resulting histogram in \cref{fig:histogram-svs-fmnist} shows the number of samples that appear as support vectors between $1$ and $100$ times. We observe that many samples consistently appear as support vectors across runs, indicating their importance regardless of the inherent symmetries in FE training.

Using this insight, we rank samples by their frequency of being selected as support vectors. Let $\freq$ denote the frequency for each sample $\sample$. Define $\topk$ to be the subset of the $\kvar$ most frequently occurring support vectors (i.e., core samples), and $\nontopk$ the complement, such that $\vert \topk \vert = \kvar$, $\vert \nontopk\vert = \vert\D\vert -\kvar$, and for all $\sample \in \topk$ and $\sample[\sampleidx'] \in \nontopk$, it holds that $\freq[\sampleidx] \geq \freq[\sampleidx']$.

\begin{table}[t]
\setlength{\tabcolsep}{3pt}
\centering
\caption{Average test accuracies across $100$ runs under different unlearning settings and values of $k \in \{10, 20\} \cdot 10^3$.}
\label{tab:svm-accuracy-settings}
\begin{tabular}{|l|c|c|c||c|c|c|}
\hline
\textbf{Unlearned} & 
\rotatebox{0}{$\topk[10]$} & 
\rotatebox{0}{$\topk[10] \cup \Dr[10k]$} & 
\rotatebox{0}{$\nontopk[10]$} & 
\rotatebox{0}{$\topk[20]$} & 
\rotatebox{0}{$\topk[20] \cup \Dr[10k]$} & 
\rotatebox{0}{$\nontopk[20]$} \\
\hline
No & 0.8971 & 0.8994 & 0.8994 & 0.8997 & 0.8987 & 0.9040 \\
SVM & 0.8964 & 0.8979 & 0.7572 & 0.8932 & 0.8934 & 0.9038 \\
FE+SVM & 0.8874 & 0.8868 & 0.4289 & 0.8601 & 0.8591 & 0.8851 \\
FE & 0.8857 & 0.8863 & 0.8289 & 0.8706 & 0.8641 & 0.8881 \\
\hline
\end{tabular}
\vspace{-.3cm}
\end{table}

With this ranking in place, we compare the impact on the components $\fe$ and $\svmmodel$ in the architecture $\fe \circ \svmmodel$ when it comes to unlearning different subsets of the training samples of different sizes. Thereby, answering the question of which and how many samples to select as core samples for training the FE. We first train a full model $\totalmodel$ on the complete dataset $\D$ using \cref{alg:training_process}, yielding a trained feature extractor $\fe$ and SVM $\svmmodel$. We then identify a subset $\Df \subset \sv$ of support vectors to be unlearned.

We explore the following unlearning scenarios:
\begin{itemize}
    \item \textbf{SVM unlearning (\unlearnedsvm):} Keeping the FE fixed, we retrain the SVM on all remaining feature embeddings $\embedding$, where $\sampleidx \in \D \setminus \Df$.
    \item \textbf{Full model unlearning (\unlearnedmodel):} We remove $\Df$ from $\D$ and retrain both the FE and SVM using \cref{alg:training_process}, yielding a new model $\fe^\prime \circ \svmmodel^\prime$ that no longer includes information about $\Df$.
    \item \textbf{FE-only unlearning (\unlearnedfe):} We retrain the FE on $\D \setminus \Df$ to obtain $\fe^\prime$, and then train a new SVM $\svmmodel^\star$ on the entire dataset (i.e., $\svmmodel^\star$ is trained on embeddings $\embedding^\prime = \fe^\prime(\sample), \sample \in \D$).
\end{itemize}

As before, the unlearning processes are repeated over $100$ trials with different random seeds. We consider two configurations for selecting $\Df$, with $\kvar \in \{10, 20\} \cdot 10^3$:
\begin{enumerate}
    \item $\Df = \topk$: unlearning the $k$ most important samples,
    \item $\Df = \nontopk$: unlearning the $\nsamples - k$ least important samples,
    \item $\Df = \topk \cup \Dr$: unlearning the $k$ most important samples plus $10^4$ randomly selected samples $\Dr$ from $\nontopk$, provided that $\nsamples - \kvar > 10^4$.
\end{enumerate}

The results are presented in \cref{tab:svm-accuracy-settings}. We observe that:
\begin{itemize}
    \item Unlearning $\topk$ samples has limited impact on accuracy when $\kvar = 10^4$, but becomes more significant for $\kvar = 20 \cdot 10^3$. Unlearning $\Dr$ in addition to $\topk$ has negligible impact.
    \item Removing $40 \cdot 10^3$ of the least important samples has only a minor effect, but removing $50 \cdot 10^3$ least important samples causes a large drop in accuracy (down to $0.4289$ for FE+SVM).
    \item Notably, if we train the FE using only the $10^4$ most important samples and train the SVM on all data, performance improves significantly—from $0.4289$ to $0.8289$.
\end{itemize}

This significant accuracy boost indicates that while the SVM benefits from access to the full dataset, the FE can be trained on a much smaller core set without major performance degradation. Since unlearning the FE is often the most computationally expensive step, this insight enables an efficient, unlearning-aware strategy. Specifically, for any $\Df \subseteq \nontopk[10^4]$, this strategy supports \emph{exact} unlearning as defined in \cref{def:relaxed_exact_unlearning}.

\begin{algorithm}[t]
\caption{Membership Inference Attack (MIA) via Cross-Entropy Confidence Metric}
\label{alg:mia}
\begin{algorithmic}[1]
\REQUIRE $\confidence[\D]$, $\confidence[\Dt]$, query samples $\Df$
\STATE Construct membership dataset $\mathcal{M} \define \{(\confidence[\D \setminus \Df], \mathbf{1}), (\confidence[\Dt], \mathbf{0})\}$
\STATE Compute ROC and obtain $\fpr_\thres, \tpr_\thres \gets \text{ROC}(\mathcal{M})$ for thresholds $\thres \in \mathcal{T}$
\STATE Optimize threshold $\optthres = \arg\max_{\thres \in \mathcal{T}} \tpr_\thres - \fpr_\thres$
\STATE Predict membership of queried samples $\mathcal{H} = \confidence[\Df] < \optthres$
\RETURN Return ``member'' if $\mathcal{H} == \text{true}$
\end{algorithmic}
\end{algorithm}

\subsection{Unlearning-Aware Model Training} \label{sec:unlearning_aware_model}
We now introduce our unlearning strategy, \algname, which is inspired by the observations detailed above. Assume access to a ranking of sample importance that identifies the top $k$ most influential training samples, denoted as $\topk$, such as the ranking method described previously. Prior to any unlearning requests, we train the feature extractor $\fe$ exclusively on the core set $\topk$, for some choice of $\kvar > 0$, using the procedure outlined in \cref{alg:training_process}. 

Once the feature extractor $\fe$ has been trained, we obtain feature embeddings for the full dataset $\D$ and use these to train a linear SVM $\svmmodel^\star$. For future reference, let $\sv$ denote the set of support vectors resulting from training $\svmmodel^\star$. The overall procedure results in the final model architecture $\fe \circ \svmmodel^\star$, which, as demonstrated in \cref{tab:svm-accuracy-settings}, achieves competitive performance despite training the feature extractor on only a subset of the data. We choose $\kvar = 20 \cdot 10^3$ in our experiments.

\subsection{Exact and Approximate Unlearning} \label{sec:exact_and_approximate_unlearning}

Let us now discuss how well unlearning works for our proposed strategy \algname. 
Depending on the nature of the requested unlearning data $\Df$, we distinguish the following two scenarios. When $\Df \cap \topk = \emptyset$, exact unlearning guarantees can be given; otherwise, we conduct approximate unlearning.

\paragraph{Exact Unlearning} For any unlearning request $\Df \subseteq \D \setminus (\sv \cup \topk)$, the model is per design exactly unlearned in the sense of \cref{def:relaxed_exact_unlearning}. This is because the feature extractor $\fe$ is entirely independent of $\D \setminus \topk$, and the SVM $\svmmodel^\star$ depends only on the support vectors $\sv$.

Moreover, since $\topk$ consists of samples frequently selected as support vectors across multiple runs (see above), we expect substantial overlap between $\sv$ and $\topk$. Consequently, a significant portion of the dataset $\D$ is eligible for exact unlearning. For example, in our experiments with $\topk[20 \cdot 10^3]$, we observed an average of $\vert \sv \cup \topk \vert = 20.5 \cdot 10^3$ over $20$ independent runs.

For an unlearning request $\Df \subset \sv \setminus \topk$, i.e., when the data to be unlearned appears only in the SVM but not in the feature extractor (FE), our method provides an efficient and exact unlearning protocol. In this case, only the final layer—here, the SVM—needs to be retrained. This is significantly less costly than retraining the FE, particularly for complex models, while still offering exact unlearning guarantees.

Referring to \cref{def:relaxed_exact_unlearning}, \algname trains the initial model $\alg(\D)$ by running \cref{alg:training_process} on $\topk$ to obtain $\fe$, and subsequently training the SVM on all of $\D$. Upon an unlearning request, the procedure reuses $\fe$ and retrains the SVM on the remaining data $\Dp = \D \setminus \Df$. This is functionally equivalent to running $\algvar(\Dp)$, which applies \cref{alg:training_process} on $\topk \subset \Dp$ to obtain $\fe$, and then trains a new SVM $\svmmodel^\star$ on $\Dp$. Since the composition of $\alg$ and $\unlearn$ is equivalent to $\algvar$, and $\algvar$ is independent of $\Df$, the generalized exact unlearning criterion is satisfied. The same logic holds even when the SVM is retrained on a subset $\Dp \subset \D \setminus \Df$, which can further reduce the computational costs of unlearning.

\begin{figure}
    \centering
    \resizebox{\linewidth}{!}{\begin{tikzpicture}

\definecolor{color0}{rgb}{0.12156862745098,0.466666666666667,0.705882352941177}
\definecolor{color1}{rgb}{1,0.498039215686275,0.0549019607843137}

\begin{axis}[
legend cell align={left},
legend style={fill opacity=0.8, draw opacity=1, text opacity=1, at={(0.03,0.97)}, anchor=north west, draw=white!80!black},
tick align=outside,
tick pos=left,
x grid style={white!69.0196078431373!black},
xlabel={Number of runs \(\displaystyle x\) a sample is claimed unlearned},
xmin=-1.43, xmax=21.45,
xtick style={color=black},
y grid style={white!69.0196078431373!black},
ylabel={Normalized number of samples},
ymajorgrids,
ymin=0, ymax=0.6130425,
ytick style={color=black},
ytick={0,0.1,0.2,0.3,0.4,0.5,0.6,0.7},
yticklabels={0.0,0.1,0.2,0.3,0.4,0.5,0.6,0.7}
]
\draw[draw=none,fill=color0] (axis cs:-0.39,0) rectangle (axis cs:0.01,0.03445);
\addlegendimage{ybar,ybar legend,draw=none,fill=color0};
\addlegendentry{Perfect unlearning}

\draw[draw=none,fill=color0] (axis cs:0.61,0) rectangle (axis cs:1.01,0.01775);
\draw[draw=none,fill=color0] (axis cs:1.61,0) rectangle (axis cs:2.01,0.01295);
\draw[draw=none,fill=color0] (axis cs:2.61,0) rectangle (axis cs:3.01,0.01275);
\draw[draw=none,fill=color0] (axis cs:3.61,0) rectangle (axis cs:4.01,0.01135);
\draw[draw=none,fill=color0] (axis cs:4.61,0) rectangle (axis cs:5.01,0.0112);
\draw[draw=none,fill=color0] (axis cs:5.61,0) rectangle (axis cs:6.01,0.0122);
\draw[draw=none,fill=color0] (axis cs:6.61,0) rectangle (axis cs:7.01,0.01285);
\draw[draw=none,fill=color0] (axis cs:7.61,0) rectangle (axis cs:8.01,0.0119);
\draw[draw=none,fill=color0] (axis cs:8.61,0) rectangle (axis cs:9.01,0.01175);
\draw[draw=none,fill=color0] (axis cs:9.61,0) rectangle (axis cs:10.01,0.0137);
\draw[draw=none,fill=color0] (axis cs:10.61,0) rectangle (axis cs:11.01,0.01565);
\draw[draw=none,fill=color0] (axis cs:11.61,0) rectangle (axis cs:12.01,0.01505);
\draw[draw=none,fill=color0] (axis cs:12.61,0) rectangle (axis cs:13.01,0.0166);
\draw[draw=none,fill=color0] (axis cs:13.61,0) rectangle (axis cs:14.01,0.0181);
\draw[draw=none,fill=color0] (axis cs:14.61,0) rectangle (axis cs:15.01,0.02125);
\draw[draw=none,fill=color0] (axis cs:15.61,0) rectangle (axis cs:16.01,0.02565);
\draw[draw=none,fill=color0] (axis cs:16.61,0) rectangle (axis cs:17.01,0.0295);
\draw[draw=none,fill=color0] (axis cs:17.61,0) rectangle (axis cs:18.01,0.0428);
\draw[draw=none,fill=color0] (axis cs:18.61,0) rectangle (axis cs:19.01,0.0687);
\draw[draw=none,fill=color0] (axis cs:19.61,0) rectangle (axis cs:20.01,0.58385);
\draw[draw=none,fill=color1] (axis cs:0.00999999999999998,0) rectangle (axis cs:0.41,0.007);
\addlegendimage{ybar,ybar legend,draw=none,fill=color1};
\addlegendentry{\algname}

\draw[draw=none,fill=color1] (axis cs:1.01,0) rectangle (axis cs:1.41,0.00895);
\draw[draw=none,fill=color1] (axis cs:2.01,0) rectangle (axis cs:2.41,0.0103);
\draw[draw=none,fill=color1] (axis cs:3.01,0) rectangle (axis cs:3.41,0.01055);
\draw[draw=none,fill=color1] (axis cs:4.01,0) rectangle (axis cs:4.41,0.01295);
\draw[draw=none,fill=color1] (axis cs:5.01,0) rectangle (axis cs:5.41,0.0131);
\draw[draw=none,fill=color1] (axis cs:6.01,0) rectangle (axis cs:6.41,0.0147);
\draw[draw=none,fill=color1] (axis cs:7.01,0) rectangle (axis cs:7.41,0.0144);
\draw[draw=none,fill=color1] (axis cs:8.01,0) rectangle (axis cs:8.41,0.0183);
\draw[draw=none,fill=color1] (axis cs:9.01,0) rectangle (axis cs:9.41,0.0174);
\draw[draw=none,fill=color1] (axis cs:10.01,0) rectangle (axis cs:10.41,0.0205);
\draw[draw=none,fill=color1] (axis cs:11.01,0) rectangle (axis cs:11.41,0.0199);
\draw[draw=none,fill=color1] (axis cs:12.01,0) rectangle (axis cs:12.41,0.0229);
\draw[draw=none,fill=color1] (axis cs:13.01,0) rectangle (axis cs:13.41,0.0257);
\draw[draw=none,fill=color1] (axis cs:14.01,0) rectangle (axis cs:14.41,0.02745);
\draw[draw=none,fill=color1] (axis cs:15.01,0) rectangle (axis cs:15.41,0.03045);
\draw[draw=none,fill=color1] (axis cs:16.01,0) rectangle (axis cs:16.41,0.0349);
\draw[draw=none,fill=color1] (axis cs:17.01,0) rectangle (axis cs:17.41,0.04485);
\draw[draw=none,fill=color1] (axis cs:18.01,0) rectangle (axis cs:18.41,0.0547);
\draw[draw=none,fill=color1] (axis cs:19.01,0) rectangle (axis cs:19.41,0.09065);
\draw[draw=none,fill=color1] (axis cs:20.01,0) rectangle (axis cs:20.41,0.50035);
\end{axis}

\end{tikzpicture}}
    \caption{Unlearning $\topk[20 \cdot 10^3]$ over $20$ runs $x \in \{1, \cdots, 20\}$.}
    \label{fig:unlearning_top_svs}
    \vspace{-.3cm}
\end{figure}
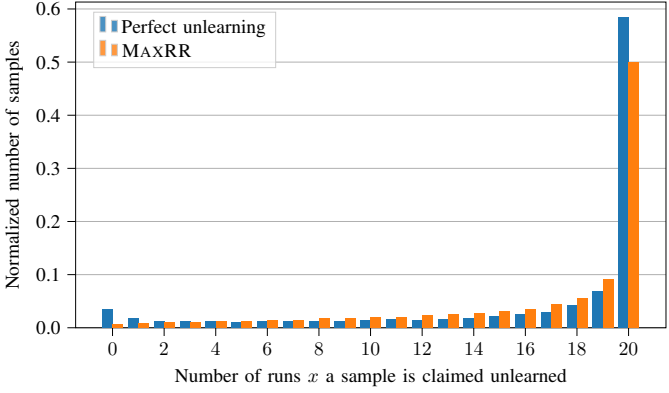

\paragraph{Approximate Unlearning}

A more critical case arises when $\Df \subset \topk$, meaning the samples to be unlearned are part of the FE's training data. In this case, we apply the same SVM-only unlearning strategy. While effective in practice, this no longer satisfies the formal guarantees of \cref{def:relaxed_exact_unlearning}.

To evaluate the success of this approximate unlearning, we apply a verification procedure $\verify[\cdot]$ based on an MIA. Since full retraining of the FE is avoided, we rely on black-box validation to assess whether $\Df$'s influence on the final model has been sufficiently removed.

Our verification method uses Platt scaling to project the SVM output onto the probability simplex, enabling the use of confidence-based MIA metrics such as cross-entropy. Given ground-truth labels and model predictions, we compute cross-entropy losses on training data $\D$ and independent test data $\Dt$ (with $10 \cdot 10^3$ samples from Fashion MNIST), which we refer to as confidences $\confidence[\D]$ and $\confidence[\Dt]$. Assuming test data corresponds to non-members and training data to members (excluding $\Df$), we optimize a threshold $\optthres$ that best separates the two distributions based on true and false positive rates. The membership of a query sample in $\Df$ is then inferred by comparing its confidence score to this threshold. The entire procedure is summarized in \cref{alg:mia}.

We repeat this process over $20$ training runs with different seeds to account for randomness in training. The FE is trained on $\topk[20 \cdot 10^3]$, and the SVM is trained on $\D$. We then simulate an unlearning request for $\Df = \topk[20 \cdot 10^3]$, and apply our strategy by retraining the SVM on $\Dp = \D \setminus \Df$, while keeping the original $\fe$. We evaluate how often the MIA identifies samples in $\Df$ as unlearned across runs and compare this with a fully retrained model (i.e., unlearned FE and SVM). The results, shown in \cref{fig:unlearning_top_svs}, demonstrate negligible differences between the two approaches.

Although this suggests that our strategy performs well under black-box auditing (input-output access only), we caution that MIA-based verification is limited and not always reliable. Nevertheless, the high similarity in verification results implies that the contribution of $\Df$ may be effectively removed through our SVM retraining. Interestingly, the model obtained via \algname achieves a higher average accuracy ($0.879795$) compared to the exact retraining baseline ($0.86252$), due to the retained utility of $\topk$ in the FE.

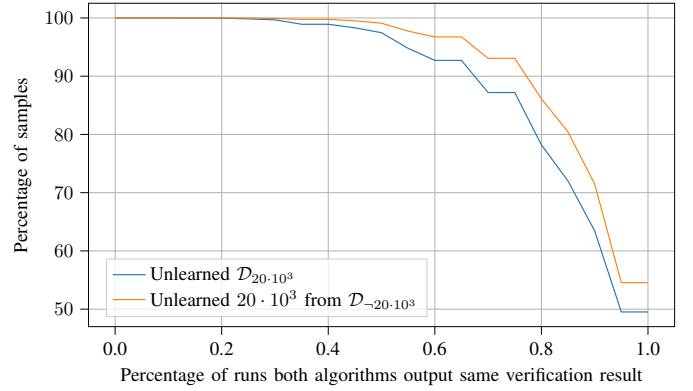
\begin{figure}[t]
    \centering
    \resizebox{\linewidth}{!}{\begin{tikzpicture}

\definecolor{color0}{rgb}{0.12156862745098,0.466666666666667,0.705882352941177}
\definecolor{color1}{rgb}{1,0.498039215686275,0.0549019607843137}

\begin{axis}[
legend cell align={left},
legend style={fill opacity=0.8, draw opacity=1, text opacity=1, at={(0.03,0.03)}, anchor=south west, draw=white!80!black},
tick align=outside,
tick pos=left,
x grid style={white!69.0196078431373!black},
xlabel={Percentage of runs both algorithms output same verification result},
xmajorgrids,
xmin=-0.05, xmax=1.05,
xtick style={color=black},
xtick={-0.2,0,0.2,0.4,0.6,0.8,1,1.2},
xticklabels={−0.2,0.0,0.2,0.4,0.6,0.8,1.0,1.2},
y grid style={white!69.0196078431373!black},
ylabel={Percentage of samples},
ymajorgrids,
ymin=46.996, ymax=102.524,
ytick style={color=black}
]
\addplot [semithick, color0]
table {%
0 100
0.05 100
0.1 100
0.15 99.985
0.2 99.955
0.25 99.845
0.3 99.675
0.35 98.915
0.4 98.915
0.45 98.305
0.5 97.465
0.55 94.78
0.6 92.715
0.65 92.715
0.7 87.195
0.75 87.195
0.8 78.215
0.85 72
0.9 63.425
0.95 49.52
1 49.52
};
\addlegendentry{Unlearned $\mathcal{D}_{20 \cdot 10^3}$}
\addplot [semithick, color1]
table {%
0 100
0.05 100
0.1 100
0.15 100
0.2 99.995
0.25 99.985
0.3 99.91
0.35 99.755
0.4 99.755
0.45 99.515
0.5 99.09
0.55 97.74
0.6 96.745
0.65 96.745
0.7 93.065
0.75 93.065
0.8 86.12
0.85 80.425
0.9 71.475
0.95 54.55
1 54.55
};
\addlegendentry{Unlearned $20 \cdot 10^3$ from $\mathcal{D}_{\neg 20 \cdot 10^3}$}
\end{axis}

\end{tikzpicture}}
    \caption{Comparison of verification results for samples $\Df$.}
    \label{fig:comparison_unlearning_strategies}
    \vspace{-.3cm}
\end{figure}

To further support our claims, we compare verification results between our method and the exact retraining baseline over two unlearning scenarios: (i) removing $\topk[20 \cdot 10^3]$, and (ii) removing a random subset of $20 \cdot 10^3$ samples from $\nontopk[20 \cdot 10^3]$ (see \cref{app:unlearn_non_core} for details). For each sample in $\Df$, we count in how many of the $20$ runs both strategies lead to the same MIA classification. \cref{fig:comparison_unlearning_strategies} shows that roughly $80\%$ of the samples were classified identically in $80\%$ of the runs, and roughly $50\%$ of the samples were identically classified in all runs—further validating our approach.

Lastly, in scenarios requiring stronger unlearning guarantees, our method is fully compatible with other established unlearning techniques. Thus, if stricter guarantees are needed, \algname can serve as a preprocessing step or be integrated with more rigorous strategies to remove any residual dependencies on $\Df$.

\bibliographystyle{IEEEtran}
\bibliography{refs}

\appendix

We provide a brief outline for the appendix. In \cref{sec:sensitivity}, we provide a study on the sensitivity of the FE and the SVM with respect to unlearning random support vectors, showing that the FE is more sensitive than the SVM to unlearning samples deemed important by the SVM. In \cref{app:unlearn_non_core}, we provide details on the unlearning of $20 \cdot 10^3$ non-core samples for which \algname provides exact unlearning guarantees. %
Complementary to the experiments in \cref{sec:exact_and_approximate_unlearning}, we further provide in \cref{app:unlearn_small_sample_sets} examples on unlearning only $2000$ core or non-core samples, instead of $20 \cdot 10^3$, showing comparable results. The details on the model architectures and dataset used are provided in \cref{app:architecture}.

\subsection{Sensitivity Analysis for FE and SVM \label{sec:sensitivity}}
Similar to \cref{sec:main_section}, we investigate the sensitivity of the individual components $\fe$ and $\svmmodel$ in the architecture $\fe \circ \svmmodel$ with respect to the removal of certain samples from the training set. Therefore, we train a model $\totalmodel$ using \cref{alg:training_process} on all training data $\D$, thereby obtaining a feature extractor $\fe$ and an SVM $\svmmodel$. We identify the support vectors $\sv$ of $\svmmodel$ and select a certain fraction of samples $\Df \subset \sv$ to be unlearned. Fixing the FE, we retrain the SVM on all embeddings $\embedding, \sampleidx \in \D \setminus \Df$. Similarly, we remove the samples $\Df$ from the training set $\D$ and retrain the FE and the SVM using \cref{alg:training_process}, thereby observing a full unlearned model $\fe^\prime \circ \svmmodel^\prime$ not containing any information about the samples in $\Df$.
Lastly, we take the unlearned FE $\fe^\prime$ and use it to train an SVM $\svmmodel^\star$ on the entire training data, i.e., trained on the embeddings $\embedding^\prime = \fe(\sample), \sampleidx \in \D$. %

We again repeat this process for $100$ rounds using different seeds for the training and the selection of $\Df$ out of the support vectors, and compare the resulting model accuracies on an independent test set. Unlearning the samples from the SVM diminishes the accuracy on Fashion MNIST by only a negligible amount, i.e., from $0.9000$ to $0.8982$ on average.
Unlearning the samples $\Df$ from the FE has a more noticeable impact, going from $0.9000$ to $0.8882$ when the SVM $\svmmodel^\star$ on the last layer is trained on all the data. For $\fe^\prime$, additionally removing the samples $\Df$ from the last layer, i.e., using the model $\svmmodel^\prime$, has only a limited effect on the performance and results in average accuracies of $0.8871$. We conclude that removing samples selected as support vectors by the SVM in the initial training has a larger impact on the FE than on the SVM. We made use of this observation to propose an efficient unlearning strategy by adapting the initial training algorithm.

\subsection{Verification of Unlearned Non-Core Samples} \label{app:unlearn_non_core}

In addition to the verification experiments in \cref{sec:unlearning_aware_model} for the unlearning of core samples, we study the verification results for cases in which we are asked to unlearn $20\cdot 10^3$ samples $\Df$ from the less important samples, i.e., $\Df \subset \nontopk[20\cdot 10^3], \vert \Df \vert = 20\cdot 10^3$. Although our unlearning strategy achieves exact unlearning guarantees according to \cref{def:relaxed_exact_unlearning}, we analyze the results from the MIA-based verification in \cref{fig:unlearning_nontop_svs}. To improve the MIA, we use only the samples $\D \setminus \topk[20\cdot 10^3] \setminus \Df$ to construct the membership dataset to find the optimal threshold $\optthres$. That is, we exclude the core samples from optimizing the MIA, since they are not representative of the samples $\Df$. The set $\D \setminus \topk[20\cdot 10^3] \setminus \Df$ reflects better the nature of the unlearned samples $\Df$ and thereby provides more reliable MIA results. However, it can be found that for both unlearning strategies, the verification method cannot reliably claim that the samples were unlearned, even though they satisfy the notion of exact unlearning, i.e., they are neither contained in the FE nor the SVM. The reason is that the samples in $\nontopk$ are often further from the decision boundary and easier to predict, hence resulting in large confidences independently of whether or not they are being used in the training. The accuracies are $0.89678$ for exact unlearning, and $0.88835$ for \algname.

\begin{figure}[t]
    \centering
    \resizebox{\linewidth}{!}{\begin{tikzpicture}

\definecolor{color0}{rgb}{0.12156862745098,0.466666666666667,0.705882352941177}
\definecolor{color1}{rgb}{1,0.498039215686275,0.0549019607843137}

\begin{axis}[
legend cell align={left},
legend style={fill opacity=0.8, draw opacity=1, text opacity=1, draw=white!80!black},
tick align=outside,
tick pos=left,
x grid style={white!69.0196078431373!black},
xlabel={Number of runs \(\displaystyle x\) a sample is claimed unlearned},
xmin=-1.43, xmax=21.45,
xtick style={color=black},
y grid style={white!69.0196078431373!black},
ylabel={Normalized number of samples},
ymajorgrids,
ymin=0, ymax=0.67809,
ytick style={color=black},
ytick={0,0.1,0.2,0.3,0.4,0.5,0.6,0.7},
yticklabels={0.0,0.1,0.2,0.3,0.4,0.5,0.6,0.7}
]
\draw[draw=none,fill=color0] (axis cs:-0.39,0) rectangle (axis cs:0.01,0.6458);
\addlegendimage{ybar,ybar legend,draw=none,fill=color0};
\addlegendentry{Perfect unlearning}

\draw[draw=none,fill=color0] (axis cs:0.61,0) rectangle (axis cs:1.01,0.09095);
\draw[draw=none,fill=color0] (axis cs:1.61,0) rectangle (axis cs:2.01,0.05135);
\draw[draw=none,fill=color0] (axis cs:2.61,0) rectangle (axis cs:3.01,0.03605);
\draw[draw=none,fill=color0] (axis cs:3.61,0) rectangle (axis cs:4.01,0.028);
\draw[draw=none,fill=color0] (axis cs:4.61,0) rectangle (axis cs:5.01,0.02445);
\draw[draw=none,fill=color0] (axis cs:5.61,0) rectangle (axis cs:6.01,0.01905);
\draw[draw=none,fill=color0] (axis cs:6.61,0) rectangle (axis cs:7.01,0.0174);
\draw[draw=none,fill=color0] (axis cs:7.61,0) rectangle (axis cs:8.01,0.01385);
\draw[draw=none,fill=color0] (axis cs:8.61,0) rectangle (axis cs:9.01,0.01185);
\draw[draw=none,fill=color0] (axis cs:9.61,0) rectangle (axis cs:10.01,0.00975);
\draw[draw=none,fill=color0] (axis cs:10.61,0) rectangle (axis cs:11.01,0.0093);
\draw[draw=none,fill=color0] (axis cs:11.61,0) rectangle (axis cs:12.01,0.00925);
\draw[draw=none,fill=color0] (axis cs:12.61,0) rectangle (axis cs:13.01,0.00855);
\draw[draw=none,fill=color0] (axis cs:13.61,0) rectangle (axis cs:14.01,0.00635);
\draw[draw=none,fill=color0] (axis cs:14.61,0) rectangle (axis cs:15.01,0.0063);
\draw[draw=none,fill=color0] (axis cs:15.61,0) rectangle (axis cs:16.01,0.00375);
\draw[draw=none,fill=color0] (axis cs:16.61,0) rectangle (axis cs:17.01,0.0031);
\draw[draw=none,fill=color0] (axis cs:17.61,0) rectangle (axis cs:18.01,0.00225);
\draw[draw=none,fill=color0] (axis cs:18.61,0) rectangle (axis cs:19.01,0.0018);
\draw[draw=none,fill=color0] (axis cs:19.61,0) rectangle (axis cs:20.01,0.00085);
\draw[draw=none,fill=color1] (axis cs:0.00999999999999998,0) rectangle (axis cs:0.41,0.5563);
\addlegendimage{ybar,ybar legend,draw=none,fill=color1};
\addlegendentry{\algname}

\draw[draw=none,fill=color1] (axis cs:1.01,0) rectangle (axis cs:1.41,0.11915);
\draw[draw=none,fill=color1] (axis cs:2.01,0) rectangle (axis cs:2.41,0.0638);
\draw[draw=none,fill=color1] (axis cs:3.01,0) rectangle (axis cs:3.41,0.0445);
\draw[draw=none,fill=color1] (axis cs:4.01,0) rectangle (axis cs:4.41,0.0359);
\draw[draw=none,fill=color1] (axis cs:5.01,0) rectangle (axis cs:5.41,0.02545);
\draw[draw=none,fill=color1] (axis cs:6.01,0) rectangle (axis cs:6.41,0.02325);
\draw[draw=none,fill=color1] (axis cs:7.01,0) rectangle (axis cs:7.41,0.01985);
\draw[draw=none,fill=color1] (axis cs:8.01,0) rectangle (axis cs:8.41,0.0165);
\draw[draw=none,fill=color1] (axis cs:9.01,0) rectangle (axis cs:9.41,0.01505);
\draw[draw=none,fill=color1] (axis cs:10.01,0) rectangle (axis cs:10.41,0.0147);
\draw[draw=none,fill=color1] (axis cs:11.01,0) rectangle (axis cs:11.41,0.0123);
\draw[draw=none,fill=color1] (axis cs:12.01,0) rectangle (axis cs:12.41,0.01005);
\draw[draw=none,fill=color1] (axis cs:13.01,0) rectangle (axis cs:13.41,0.0089);
\draw[draw=none,fill=color1] (axis cs:14.01,0) rectangle (axis cs:14.41,0.00785);
\draw[draw=none,fill=color1] (axis cs:15.01,0) rectangle (axis cs:15.41,0.0077);
\draw[draw=none,fill=color1] (axis cs:16.01,0) rectangle (axis cs:16.41,0.00615);
\draw[draw=none,fill=color1] (axis cs:17.01,0) rectangle (axis cs:17.41,0.00485);
\draw[draw=none,fill=color1] (axis cs:18.01,0) rectangle (axis cs:18.41,0.00375);
\draw[draw=none,fill=color1] (axis cs:19.01,0) rectangle (axis cs:19.41,0.0029);
\draw[draw=none,fill=color1] (axis cs:20.01,0) rectangle (axis cs:20.41,0.0011);
\end{axis}

\end{tikzpicture}}
    \caption{Unlearning random $20\cdot 10^3$ samples from $\nontopk[20k]$}
    \label{fig:unlearning_nontop_svs}
\end{figure}
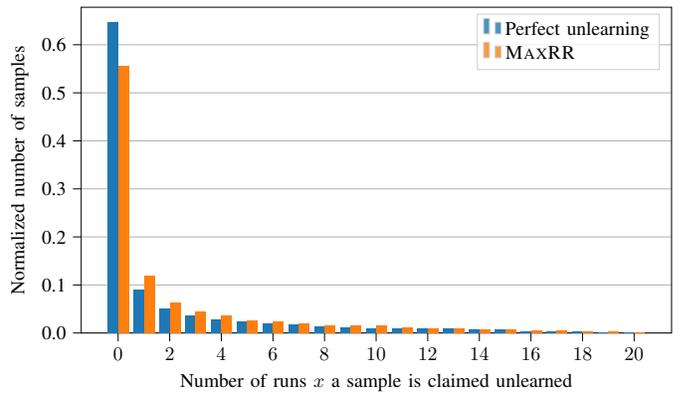

\subsection{Verification of Unlearning Smaller Fractions of Samples} \label{app:unlearn_small_sample_sets}

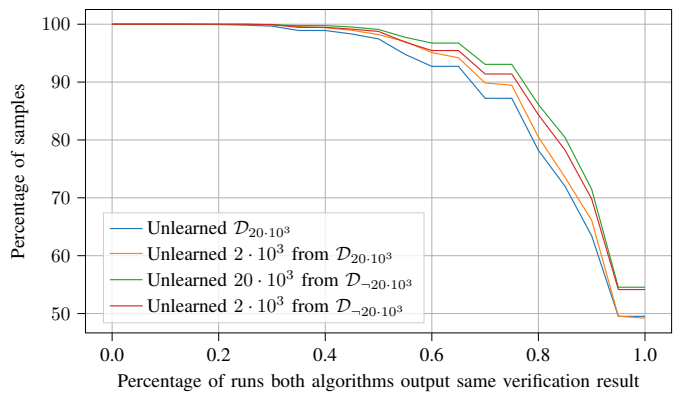
\begin{figure}[b]
    \centering
    \resizebox{\linewidth}{!}{\begin{tikzpicture}

\definecolor{color0}{rgb}{0.12156862745098,0.466666666666667,0.705882352941177}
\definecolor{color1}{rgb}{1,0.498039215686275,0.0549019607843137}
\definecolor{color2}{rgb}{0.172549019607843,0.627450980392157,0.172549019607843}
\definecolor{color3}{rgb}{0.83921568627451,0.152941176470588,0.156862745098039}

\begin{axis}[
legend cell align={left},
legend style={fill opacity=0.8, draw opacity=1, text opacity=1, at={(0.03,0.03)}, anchor=south west, draw=white!80!black},
tick align=outside,
tick pos=left,
x grid style={white!69.0196078431373!black},
xlabel={Percentage of runs both algorithms output same verification result},
xmajorgrids,
xmin=-0.05, xmax=1.05,
xtick style={color=black},
xtick={-0.2,0,0.2,0.4,0.6,0.8,1,1.2},
xticklabels={−0.2,0.0,0.2,0.4,0.6,0.8,1.0,1.2},
y grid style={white!69.0196078431373!black},
ylabel={Percentage of samples},
ymajorgrids,
ymin=46.66, ymax=102.54,
ytick style={color=black}
]
\addplot [semithick, color0]
table {%
0 100
0.05 100
0.1 100
0.15 99.985
0.2 99.955
0.25 99.845
0.3 99.675
0.35 98.915
0.4 98.915
0.45 98.305
0.5 97.465
0.55 94.78
0.6 92.715
0.65 92.715
0.7 87.195
0.75 87.195
0.8 78.215
0.85 72
0.9 63.425
0.95 49.52
1 49.52
};
\addlegendentry{Unlearned $\mathcal{D}_{20 \cdot 10^3}$}
\addplot [semithick, color1]
table {%
0 100
0.05 100
0.1 100
0.15 99.95
0.2 99.95
0.25 99.95
0.3 99.9
0.35 99.6
0.4 99.4
0.45 98.95
0.5 98.2
0.55 97
0.6 95.1
0.65 94.2
0.7 89.85
0.75 89.45
0.8 80.5
0.85 73.55
0.9 66.1
0.95 49.55
1 49.2
};
\addlegendentry{Unlearned $2 \cdot 10^3$ from $\mathcal{D}_{20 \cdot 10^3}$}
\addplot [semithick, color2]
table {%
0 100
0.05 100
0.1 100
0.15 100
0.2 99.995
0.25 99.985
0.3 99.91
0.35 99.755
0.4 99.755
0.45 99.515
0.5 99.09
0.55 97.74
0.6 96.745
0.65 96.745
0.7 93.065
0.75 93.065
0.8 86.12
0.85 80.425
0.9 71.475
0.95 54.55
1 54.55
};
\addlegendentry{Unlearned $20 \cdot 10^3$ from $\mathcal{D}_{\neg 20 \cdot 10^3}$}
\addplot [semithick, color3]
table {%
0 100
0.05 100
0.1 100
0.15 100
0.2 100
0.25 100
0.3 99.95
0.35 99.45
0.4 99.45
0.45 99.15
0.5 98.75
0.55 96.9
0.6 95.45
0.65 95.45
0.7 91.4
0.75 91.4
0.8 84.35
0.85 78.25
0.9 69.8
0.95 54.15
1 54.15
};
\addlegendentry{Unlearned $2 \cdot 10^3$ from $\mathcal{D}_{\neg 20 \cdot 10^3}$}
\end{axis}

\end{tikzpicture}}
    \caption{Comparison of verification results for samples $\Df$.} %
    \label{fig:comparison_all_unlearning_strategies}
\end{figure}

In \cref{sec:main_section}, we provided results for jointly unlearning $20 \cdot 10^3$ samples, i.e., either $\topk[20 \cdot 10^3]$ or random $20 \cdot 10^3$ samples from $\nontopk[20 \cdot 10^3]$. We additionally study a less extreme and more realistic scenario of unlearning $2 \cdot 10^3$ samples, being still a relatively large fraction $\frac{1}{30}$ of the total dataset. In practice, the unlearning request would likely comprise even fewer samples. We investigate two cases, where $2 \cdot 10^3$ samples are drawn randomly from $\topk[20 \cdot 10^3]$, and where $2 \cdot 10^3$ samples are drawn randomly from $\nontopk[20 \cdot 10^3]$. We perform the same analysis as for \cref{fig:comparison_unlearning_strategies} and depict the results in \cref{fig:comparison_all_unlearning_strategies}. We observe very similar behavior, almost independent of the number of samples to be unlearned and their underlying nature.

\subsection{Model Architecture} \label{app:architecture}

We use the LeNet5 architecture from \cite{lecun1998gradient}, with $61706$ model parameters. The layers of LeNet5 are provided in \cref{tab:lenet5}. The architecture of the FE $\fe$ contains everything but the last linear layer, here replaced by an SVM. For our experiments, we use the Fashion MNIST dataset comprising images of Zalando's articles with $28 \times 28$ pixels, each taking values from $0$ to $255$. The training dataset $\D$ consists of $60 \cdot 10^3$ samples, and the test dataset contains $10 \cdot 10^3$ samples.
\begin{table}[H]
\centering
\caption{LeNet5 Architecture Overview} \label{tab:lenet5}
\begin{tabular}{|l|l|l|} \hline
\textbf{Layer} & \textbf{Specification} & \textbf{Activation} \\ \hline
5x5 Conv & 6 filters, stride 1 & ReLU, AvgPool (2x2) \\
5x5 Conv & 16 filters, stride 1 & ReLU, AvgPool (2x2) \\
Linear & 120 units & ReLU \\
Linear & 84 units & ReLU \\
Linear & 10 units & Softmax \\ \hline
\end{tabular}
\end{table}

\end{document}